\documentclass[letterpaper, 10 pt, conference]{ieeeconf}  

\IEEEoverridecommandlockouts                              

\renewcommand{\baselinestretch}{0.955}

\usepackage{amsmath}
\usepackage{amssymb}
\usepackage{tikz}

\usepackage{amsthm}
\usepackage{xspace}

\usepackage{graphics} 

\newtheorem{myremark}{Remark}

\newtheorem{mylem}{Lemma}
\newtheorem{mydef}{Definition}
\newtheorem{myprop}{Proposition}
\newtheorem{mycor}{Corollary}
\newtheorem{mytheorem}{Theorem}

\newcommand{\refineCBF}{\textsc{refineCBF}\xspace}
\usepackage{xcolor}
\usepackage{accents}
\usepackage{centernot}
\usepackage{algorithm}
\usepackage[noend]{algpseudocode}
\usepackage{subcaption} 
\usepackage{mathtools}
\usepackage[font=small,labelfont=bf]{caption}

\newcommand{\LL}{\mathcal{L}}

\newcommand{\UU}{\mathcal{U}}

\newcommand{\ctrl}{u}

\newcommand{\BV}{B}
\newcommand{\BVg}{\BV_{\gamma}}
\newcommand{\BVgi}[1] {
  \BV_{\gamma_{#1}}}
\newcommand{\dBVgi}[1] {
  \dot{\BV}_{\gamma_{#1}}}
\newcommand{\gi}[1]{\gamma_{#1}}
\newcommand{\maxu}{\max_{\ctrl\in\UU_{[t,0]}}}
\newcommand{\mint}{\min_{s\in[t,0]}}

\newif\ifmargincomments
\margincommentstrue
\ifmargincomments
\else

\fi

\newif\ifsuggestions
\suggestionstrue

\ifsuggestions
\newcommand\tonkens[1]{\textcolor{cyan}{[ST: #1]}}
\newcommand\shnote[1]{\textcolor{violet}{[SH: #1]}}

\else
\newcommand\tonkens[1]{}
\newcommand\shnote[1]{}
\fi

\makeatletter
\let\NAT@parse\undefined
\makeatother
\usepackage{hyperref}

\title{\LARGE \bf
Refining Control Barrier Functions through Hamilton-Jacobi Reachability
}

\author{Sander Tonkens and Sylvia Herbert
\thanks{Sander Tonkens and Sylvia Herbert are with Department of Mechanical and Aerospace Engineering, University of California, San Diego \{\href{mailto:stonkens@ucsd.edu}{stonkens}, \href{mailto:sherbert@ucsd.edu}{sherbert}\}@ucsd.edu. This work was supported by the ONR YIP program (grant \#N00014-22-1-2292); this article solely reflects the opinions and conclusions of its authors.}
}

\begin{document}

\maketitle
\thispagestyle{empty}
\pagestyle{empty}
\begin{abstract}
Safety filters based on Control Barrier Functions (CBFs) have emerged as a practical tool for the safety-critical control of autonomous systems. These approaches encode safety through a value function and enforce safety by imposing a constraint on the time derivative of this value function. However, synthesizing a \emph{valid} CBF that is not overly conservative in the presence of input constraints is a notorious challenge. In this work, we propose refining a \emph{candidate} CBF using formal verification methods to obtain a \emph{valid} CBF. In particular, we update an expert-synthesized or backup CBF using dynamic programming (DP) based reachability analysis. Our framework, \refineCBF, guarantees that with every DP iteration the obtained CBF is provably at least as safe as the prior iteration and converges to a \emph{valid} CBF. Therefore, \refineCBF can be used in-the-loop for robotic systems. We demonstrate the practicality of our method to enhance safety and/or reduce conservativeness on a range of nonlinear control-affine systems using various CBF synthesis techniques in simulation.  
\end{abstract}

\section{Introduction}
Widespread deployment of autonomous systems, from self-driving vehicles to medical robots, hinges on their ability to perform complex tasks while providing safety assurances. The growing complexity of the scenarios these systems operate in demands that safety properties are rigorously encoded in the controller design. At its core, safety is a constraint satisfaction problem; the state of the system must never enter any pre-defined failure regions such as collisions or traffic rule violations~\cite{FisacLugovoyEtAl2019}. \emph{Ergo}, a core research problem over the years has been synthesizing a control invariant subset of the non-failure regions, i.e. a set of initial states that can remain out of the failure region using admissible control inputs for a predefined period of time -- or indefinitely.

Within a typical autonomy stack, a popular approach to enforce set invariance relies on using a \emph{safety filter}~\cite{BrunkeGreeffEtAl2021}. Safety filters minimally modify an arbitrary nominal (safety-agnostic) policy in order to keep the system's trajectory from going into a pre-defined set of obstacles. In our work we achieve this by enforcing a constraint on (i.e. filtering) the control applied at every point along the trajectory. While decoupling safety from performance in such manner can lead to suboptimal (although safe) behavior~\cite{KollerBerkenkampEtAl2018}, it offers more modularity than considering these potentially conflicting objectives jointly.

In this paper, we focus on value function-based approaches for jointly encoding and enforcing safety. Value functions endow nonlinear systems with rigorous safety guarantees and have been successfully deployed on hardware~\cite{DawsonLowenkampEtAl2022, SingletarySwannEtAl2022}. Value function-based safety approaches encode a state's relative degree of safety through its value and enforce safety via set invariance by imposing a constraint on its derivative. However, many current methods either encode safety or enforce safety in an ad-hoc manner (e.g. using hand-designed heuristic functions or data-driven functions). In this work, we assume we are given a heuristically constructed safe set defined through a Control Barrier Function (CBF) candidate. This candidate CBF candidate may cause either unsafe or overly conservative behavior. We propose updating this function with a principled constructive approach leveraging Hamilton-Jacobi (HJ) reachability analysis. Upon convergence of our proposed iterative algorithm, we obtain a CBF that is guaranteed to be safe. See Figure~\ref{fig:illustration} for a visual explanation of our algorithm.
\begin{figure}[!t]
    \vspace{0.25cm}
    \centering
    \includegraphics[width=\columnwidth]{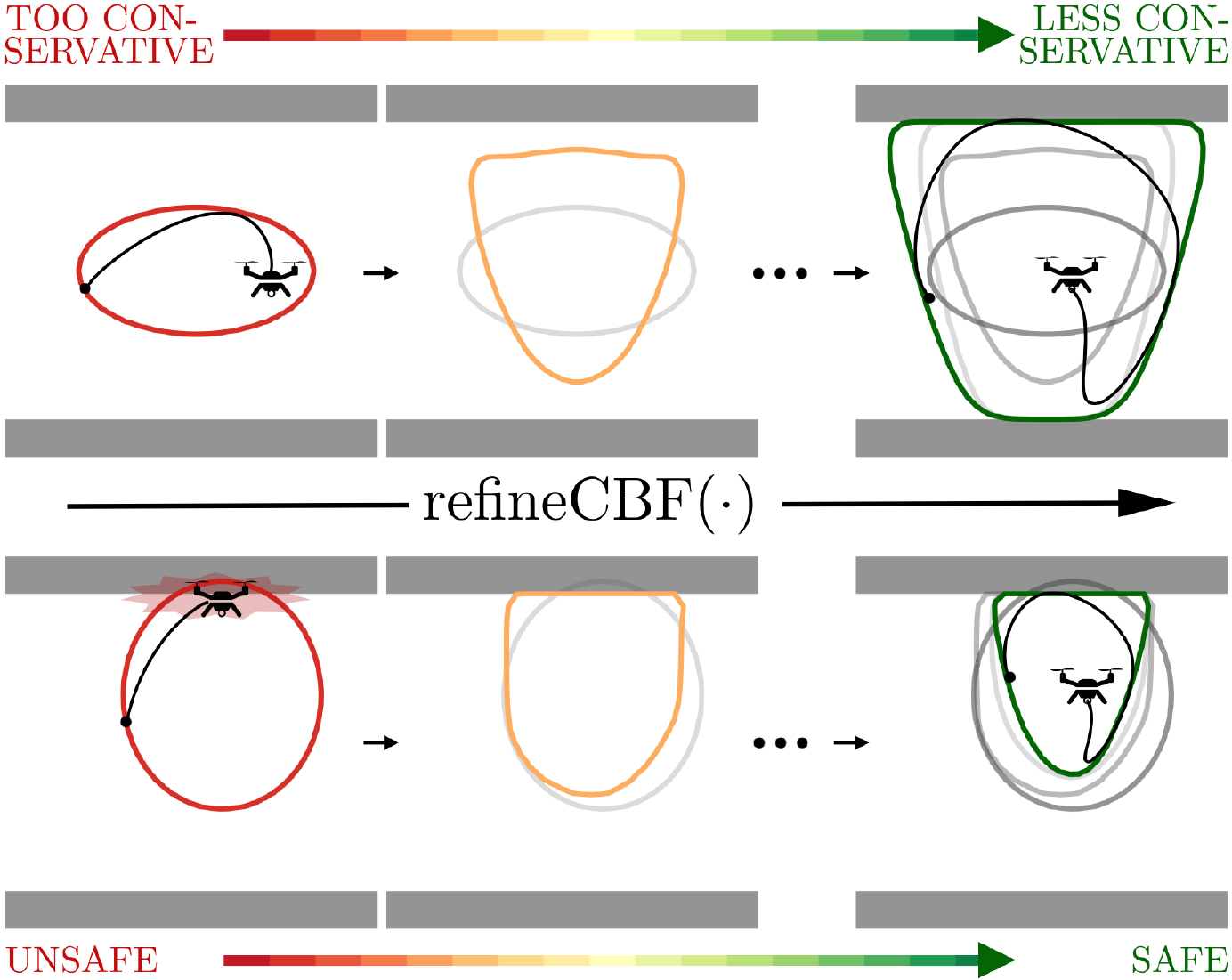}
    \caption{Illustrative depiction of our proposed \refineCBF algorithm, with an obstacle (filled gray). At low velocities (top) the candidate CBF (left) is too conservative, whereas at high velocities (bottom) it is unsafe. Our framework updates the value function until convergence. \refineCBF can be used with any CBF to enhance safety (bottom) and/or performance (top). It guarantees that the found CBVF becomes monotonically more valid with each iteration and is guaranteed to converge to a valid CBVF at convergence.}
    \label{fig:illustration}
    \vspace{-0.8cm}
\end{figure}

\textbf{Related work:} \textbf{\emph{Control Barrier Functions}} have recently become popular as a tool for maintaining safety for a wide range of autonomous systems~\cite{AmesCooganEtAl2019}. CBFs define a desired safe set as their 0-superlevel set. To maintain trajectories within this 0-superlevel set, an inequality constraint is placed on the control input through the derivative of the CBF. This condition can be principally enforced through a safety filter that solves an optimization problem online. Crucially, for control-affine systems, this inequality constraint is linear in the control input. As such, the resulting minimally invasive safety filter that maintains safety can be described as a quadratic program, which can be efficiently solved online. However, CBFs are non-constructive and hence are typically handcrafted by experts. Often these synthesized CBFs are merely \emph{candidate} CBFs, and hence do not provide persistent safety and feasibility guarantees~\cite{ZengZhangEtAl2021}. To circumvent the challenges of hand-synthesizing valid yet not overly conservative CBFs, recent work has leveraged function approximation schemes to learn CBFs, e.g., using safe expert trajectories~\cite{RobeyHuEtAl2020} or labeled state-action pair samples~\cite{DawsonQinEtAl2021}. Imitation learning-inspired approaches can lead to undesired and/or overly conservative behavior in practice, whilst sampling-based approaches struggle to provide strict safety guarantees. Additionally, these neural CBFs are trained based on a set of assumptions that might not be valid for deployment in realistic settings.

Recently, backup CBFs have been proposed as a method to address the difficulty of synthesizing a control invariant set~\cite{GurrietMoteEtAl2019}. These controllers implicitly define a set of states as safe if they are safely within reach of an appropriate backup set when using a pre-defined backup policy (e.g. maximally breaking maneuver for a car). Optimization-based backup CBF schemes~\cite{GurrietMoteEtAl2019, ChenJankovicEtAl2021} offer a principled way of enforcing invariance by imposing additional constraints on the control input along the backup trajectory. However, this method substantially increases the computational complexity of the optimization problem that defines the safety filter and, hence, is typically not suited for hardware applications~\cite{SingletarySwannEtAl2022}. In contrast, derivative-free backup CBF schemes~\cite{SingletaryGurrietEtAl2020, AbateCoogan2020} solely require evaluating the flow of the dynamics under the backup strategy. However, they can suffer from degraded performance near the boundary of the implicitly defined control invariant set and are not principally enforced using a safety filter, but as a hand-tuned trade-off between safety and performance. Additionally, derivative-free backup CBFs do not provide robustness guarantees, unlike CBFs, which enforce an exponential return to safety from unsafe states~\cite{XuTabuadaEtAl2015}.

For more complex robotic systems, synthesizing an analytic CBF or defining a good backup policy can be very complicated. For such scenarios, \textbf{\emph{Hamilton-Jacobi reachability analysis}}, a formal method, can be used to obtain a safe set. HJ-based value functions can be solved numerically using dynamic programming (DP)~\cite{MitchellBayenEtAl2005} to obtain the largest feasible control invariant set, the \emph{viability kernel}, and the optimally safe control at each state. While constructive, DP methods rely on spatial discretization of the state space and hence scale exponentially with the dimension of the system. Techniques leveraging system structure, warmstarting the DP computation, and other computational methods can significantly improve computational scalability~\cite{HerbertChoiEtAl2021}, however the curse of dimensionality remains a fundamental limitation to any DP-inspired approach. Additionally, while HJ reachability provides a constructive method to encode the largest control invariant set, it cannot be readily used with a safety filter~\cite{LeungSchmerlingEtAl2019}. 

\par\textbf{\emph{Control Barrier Value Functions}} (CBVFs) were recently introduced by Choi Et Al.~\cite{ChoiLeeEtAl2021} in an effort to realize both the constructiveness of HJ reachability and the principled online enforcement of CBFs. CBVFs recover the viability kernel and enable enforcing safety online using the CBF constraint. Yet, like HJ reachability, the computation of a CBVF requires state discretization and thus does not scale well with the state dimension. In summary, constructive methods for encoding safety (HJ reachability, CBVFs) suffer from poor scalability, whereas constructive methods for enforcing set invariance (CBFs) are often overly conservative and/or cannot guarantee safety.

\textbf{Contributions:} This work proposes refining CBFs using HJ reachability. This allows us to take advantage of hand-tuned or data-driven CBF candidates as an efficient initialization that is then rigorously refined using dynamic programming. We refer to this algorithm as \refineCBF. We show that the obtained CBVF is a valid finite-time CBVF throughout the iterations of \refineCBF that is guaranteed to not become more unsafe with every iteration. Additionally, \refineCBF recovers a valid (infinite-time) CBVF at convergence. We demonstrate how \refineCBF, can be used with both analytical and backup set based CBFs to (a) enlarge the safe set of overly conservative CBFs, improving performance, and (b) refine invalid CBFs, improving safety.\footnote{All code is available at \href{https://github.com/UCSD-SASLab}{https://github.com/UCSD-SASLab/refineCBF}}

\textbf{Organization:}
We start with a brief overview of CBFs and CBVFs in Section~\ref{sec:background}. Next, we motivate the usefullness of refining CBFs using DP-based reachability by introducing several practically relevant use cases in the context of robotics in Section~\ref{sec:relevance}. We provide the theoretical foundations for refining CBVFs from CBFs in Section~\ref{sec:theory}, followed by how \refineCBF is implemented in practice~\ref{sec:practice} We present case studies in simulation in Section~\ref{sec:results} and conclude with the merits of \refineCBF and avenues of future work in Section~\ref{sec:conclusion}.

\section{Preliminaries}\label{sec:background}
For ease of notation, we consider disturbance-free dynamics, but note that the theory below holds for bounded disturbances that draw from \emph{non-anticipative strategies} with respect to the input~\cite{EvansSouganidis1984}. We consider a control-affine system
\begin{equation}\label{eq:ol_dynamics}
    \dot{x} = f(x) + g(x)u,
\end{equation}
with state $x \in \mathcal{X} \subseteq \mathbb{R}^n$, input $u \in \mathcal{U} \subseteq \mathbb{R}^m$, and functions $f: \mathbb{R}^n \to \mathbb{R}^n$ and $g: \mathbb{R}^n \to \mathbb{R}^{n \times m}$ assumed to be locally Lipschitz continuous on their domains. Specifying the input as a control policy $\pi: \mathbb{R}^n \to \mathbb{R}^m$ that is locally Lipschitz continuous on its domain yields the closed-loop dynamics:
\begin{equation}\label{eq:cl_dynamics}
    \dot{x} = f(x) + g(x)\pi(x).
\end{equation}
We define system trajectories at time $t$ starting at $x_0$, $t_0$ under control signal $u$ by $\xi^u_{x_0, t_0}(t)$. The Lipschitz continuity of $f, g$, and $\pi$ provide a sufficient condition on the existence and uniqueness of solutions to~\eqref{eq:cl_dynamics}, which is required for most of the results discussed below to hold.

We consider a constraint set $\LL \subseteq \mathcal{X}$, defined as a 0-superlevel set of a bounded Lipschitz continuous function $\ell: \mathcal{X} \to \mathbb{R}$. In practice, given a set of failure states that the system should avoid entering, we define $\LL$ as its complement. We call $\ell(x)$ the \emph{constraint function}, which encodes a notion of distance from the failure states (e.g. a signed distance function denoting distance to obstacles). Staying safe is equivalent to assuring a trajectory remains in $\LL$.

Next, we define the viability kernel $\mathcal{S}(t)$ as the largest control invariant subset of the constraint set. $\mathcal{S}(t)\subseteq \LL$ is the set of all initial states from which there exists an admissible control signal that keeps the system inside $\LL$, i.e. that keeps the system safe, over a time duration $t$:
\begin{equation*}
\begin{split}
        \mathcal{S}(t):= \{x \in \LL: &\exists \ u(\cdot) \in \mathcal{U} \text{ s.t. } \xi_{x,t}^u(s) \in \LL \ \forall s \in [t, 0]\}.
\end{split}
\end{equation*}
We denote the infinite-time viability kernel as $\mathcal{S}^*$.

\subsection{Control Barrier Functions}
Let $\mathcal{C}_h$ denote a 0-superlevel set of the function $h: \mathbb{R}^n \to \mathbb{R}$, i.e. $\mathcal{C}_h = \{x \in \mathcal{X} \mid h(x) \geq 0\}$.
\begin{mydef}[Control Barrier Function~\cite{AmesXuEtAl2016}]
A continuously differentiable function $h$ is a Control Barrier Function for~\eqref{eq:ol_dynamics} if there exists an extended class $\mathcal{K}$ function\footnote{A continuous function $\alpha: (-b, a) \rightarrow (-\infty, \infty)$ is said to belong to extended class $\mathcal{K}$ for some $a, b > 0$ if it is strictly increasing and $\alpha(0) = 0$.} $\alpha$ such that:
\begin{equation}\label{eq:cbf}
    \sup_{u \in \mathcal{U}} \dot{h}(x):= \sup_{u\in\mathcal{U}} L_f h(x) + L_g h(x) u \geq -\alpha(h(x)).
\end{equation}
\end{mydef}

In practice, a linear function $\alpha(z) = \gamma z$ with $\gamma \in \mathbb{R}_+$ is often used, in which case $\gamma$ serves as the \emph{maximal discount rate} of $h(x)$. As such, we will consider this linear function in the remainder of the paper.
\begin{myremark}[Candidate CBF]
    We define a \textbf{candidate} CBF as a continuously differentiable function $h: \mathbb{R}^n \to \mathbb{R}$ for a closed set $\mathcal{C} \subseteq \mathcal{X}$ if it satisfies $h(x) > 0 \Leftrightarrow x \in \text{Int}(\mathcal{C}$) and $h(x) = 0 \Leftrightarrow x \in \partial \mathcal{C}$, which is necessary but not sufficient for~\eqref{eq:cbf}. For a candidate CBF to be a \textbf{valid} CBF it must additionally satisfy (a) Eq. $\eqref{eq:cbf}$ $\ \forall x \in \mathcal{C}$ and (b) $\mathcal{C}$ to be a subset of the infinite-time viability kernel $\mathcal{S}^*$.
\end{myremark}
Valid CBFs can be used to impose a constraint on the input of the system which assures the forward invariance of $\mathcal{C}_h$ and hence guarantees safety.
\begin{mycor}[Forward invariance of $\mathcal{C}_h$~\cite{AmesXuEtAl2016}]
    Let $h$ a CBF on $\mathcal{X}$, then any Lipschitz continuous policy $\pi: \mathcal{X} \to \mathcal{U}$ such that $\pi(x) \in \mathcal{K}_h(x)$ where:
    \begin{equation}\label{eq:cbf_constraintset}
        \mathcal{K}_h(x):=\{u \in \mathcal{U} \mid \dot{h}(x) \geq -\gamma h(x)\},
    \end{equation}
    will render $\mathcal{C}_h$ forward invariant.
\end{mycor}
Valid CBFs additionally provide a form of natural robustness by enforcing an exponential return to safety.
\begin{myremark}[Asymptotic stability of $\mathcal{C}_h$,~\cite{XuTabuadaEtAl2015}]\label{rem:cbf_unsafe}
    For a valid CBF $h$, $\mathcal{C}_h$ is asymptotically stable. In particular, for $h(x)~<~0$, $h$ describes a Control Lyapunov Function (CLF). While being very invasive, the CLF constraint~\eqref{eq:cbf} enforces an exponential return to safety, which is typically desired. The CBF hence benefits from natural robustness.
\end{myremark}

\subsection{Control Barrier-Value Functions}
In an effort to provide a unifying framework for CBFs and HJ reachability\footnote{We refer the reader to~\cite{BansalChenEtAl2017b} for an overview on HJ reachability.}, Choi et al.~\cite{ChoiLeeEtAl2021} introduced the Control Barrier-Value Function (CBVF), a constructive method for computing \emph{valid} CBFs:
\begin{mydef}[Control Barrier-Value Function~\cite{ChoiLeeEtAl2021}]
    A Control Barrier-Value Function $B_\gamma: \mathcal{X} \times (-\infty, 0] \to \mathbb{R}$ is defined as:
    \begin{equation}\label{eq:cbvf}
     \BVg(x,t) := \maxu \mint e^{\gamma (s-t)} \ell(\xi_{x,t}^u(s)),
    \end{equation}
    for some $\gamma \in \mathbb{R}_+$ and $\forall t\leq 0$, with initial condition $B_\gamma(x, 0) = \ell(x)$.
\end{mydef}
Compared to standard HJ reachability~\cite{BansalChenEtAl2017b}, the introduction of the exponentially increasing term $e^{\gamma (s-t)}$ in~\eqref{eq:cbvf} enables encoding the maximal decrease rate $\gamma$ in~\eqref{eq:cbf}. The standard HJ reachability problem is a special case of~\eqref{eq:cbvf}, where $\gamma=0$, i.e. $B_0$.

Defining $\mathcal{C}_{B_\gamma}(t) = \{x \mid \BVg (x,t) \geq 0\}$, we have the following important result:
\begin{myprop}[CBVF recovers the viability kernel~\cite{ChoiLeeEtAl2021}]\hspace{1cm}
$\forall t \leq 0$, $\gamma \in \mathbb{R}_+$, $\mathcal{C}_{B_\gamma} (t) = \mathcal{S}(t)$.
\end{myprop}
A CBVF hence recovers the largest (time-varying) control invariant set for maintaining safety, similar to the standard HJ value function. In practice, we spatially discretize the state space and solve for $B_\gamma$ recursively using dynamic programming on this grid~\cite{Herbert2020}:
\begin{equation}\label{eq:dp}
\begin{split}
    B_\gamma(x,t) = \max_{u \in \mathcal{U}}\min \{& \min_{s \in [t, t + \delta]} e^{\gamma (s-t)}\ell(\xi_{x,t}^u(s)), \\ & e^{\gamma \delta} B_\gamma(\xi_{x,t}^u(t + \delta), t + \delta)\},
\end{split}
\end{equation}
with initial condition $B_\gamma(x,0) = \ell(x)$.

Unlike a standard HJ value function~\cite{BansalChenEtAl2017b}, using a CBVF within a safety filter does allow a trajectory to approach the boundary of the safe set, which is typically desired for robotics applications. In particular, it employs the same derivative constraint as a CBF~\eqref{eq:cbf}, and can hence be enforced principally in an optimization-based safety filter. Similar to~\eqref{eq:cbf_constraintset}, we define the allowable control range (where for generality we consider different discount factors for the \emph{offline} construction $\gi{i}$ and \emph{online} enforcement $\gi{j}$ of the CBVF constraint) as:
\begin{equation}\label{eq:cbvf_constraintset}
    \mathcal{K}^{\gi{j}}_{\BVgi{i}}(x,s) :\left\{u \in \mathcal{U} \mid \dBVgi{i}(x,s) + \gi{j} \BVgi{i}(x,s) \geq 0 \right\}.
\end{equation}

As the discount rate and control bounds are incorporated in the construction of a CBVF $\BVg$, the online control problem is guaranteed to be feasible for the same rate and bounds, i.e. for $\gamma =\gi{i}=\gi{j}$, $\mathcal{K}_{\BVg}^{\gamma}(x,t) \neq \emptyset$ $\forall x \in \mathcal{X}$, $t \leq 0, \gamma \in \mathbb{R}_+$. Additionally, $\mathcal{C}_{B_\gamma}(t) \subseteq \mathcal{S}(t)$ hence the CBVF is a \emph{valid} CBF almost everywhere\footnote{A CBVF might have points of non-differentiability.}. Nonetheless, this approach still suffers from the same dimensionality issues as HJ reachability.

\subsection{Online active safety filter}
Following~\cite{GurrietSingletaryEtAl2018}, given a locally Lipschitz continuous \emph{safety-agnostic nominal policy} $\hat \pi (x)$ we define the minimally invasive safety filter as:
\begin{equation}\label{eq:online-cbf}
    \begin{split}
        u^*(x)= \> \arg \min \limits_{u} \quad &\lVert u - \hat \pi(x) \rVert_R^2 \\
        \text{subject to}\quad & \dot{h}(x) + \gamma h(x) \ \geq 0 \\
        & u \in \mathcal{U},
    \end{split}
\end{equation}
for a positive definite matrix $R \in S^m_+$ and a CBF or CBVF $h(x)$. The safety filter, if feasible everywhere in $\mathcal{C}_h$, guarantees its forward invariance when applying $u^*$. For control affine dynamics, the CBF constraint is linear in $u$. If additionally $\mathcal{U}$ is a convex polytope, \eqref{eq:online-cbf} is a quadratic program, which can be solved in real-time for embedded applications.

\section{Practical Implications}\label{sec:relevance}
This work proposes \refineCBF, which leverages the constructiveness of DP-based reachability analysis, the practical tools and approaches for synthesizing a \emph{candidate} CBF, and the principled online enforcement of CBFs. Compared to CBF candidates \refineCBF (upon convergence) guarantees safety, whereas compared to standard CBVFs \refineCBF enables leveraging expert knowledge and data-driven methods. In addition, \refineCBF (while converging) is at least as safe as the standard CBVF method and often converges faster (albeit empirically). 
As previously mentioned, \refineCBF can be used in combination with any synthesized \emph{candidate} CBF. We necessitate the dynamics (with optionally bounded disturbances) and the obstacles to be known, but do not require any additional assumptions. However, given that \refineCBF relies on spatially discretized DP recursion, we typically consider low-dimensional models with state dimension less than or equal to $6$.

Some exciting use cases for \refineCBF are the following:
\begin{enumerate}
    \item An invalid CBF can be updated to be valid. Throughout the updates, the updated CBVF will become less unsafe with each step (see Section~\ref{sec:theory}).
    \item An overly conservative, yet \emph{valid}, CBF can be updated to become less conservative while maintaining safety through the convergence.
    \item A CBF that was synthesized for a specific set of operating conditions, e.g., a maximum wind disturbance or a maximum control authority, can be refined to a new set of operating conditions.
    \item An implicit CBF (both derivative-free and optimization-based backup CBFs) can be rendered explicit using policy evaluation and value function refinement. This enables reducing conservativeness for derivative-free implicit CBFs~\cite{SingletarySwannEtAl2022} and reducing online computation for optimization-based implicit CBFs~\cite{GurrietMoteEtAl2019}.
\end{enumerate}
Importantly, we argue that for all of these use cases \refineCBF can be used in-the-loop, as every iteration of the DP recursion is at least as safe as the prior iteration. For high-dimensional systems, \refineCBF is limited by the memory of the system (either onboard or in-the-cloud), and less on the computational power.

\section{Refining Control Barrier Functions}\label{sec:theory}
In this work, we refine CBFs using the CBVF formulation. Concretely, this consists of initializing the DP recursion associated with the CBVF with a \emph{candidate} CBF, $h(x)$, instead of the constraint function $\ell(x)$. From an optimization lens, this can be seen as a form of warmstarting, with $h(x)$ being a better initial guess of a \emph{valid} CBF. As such, we can take advantage of tools that have been developed to quickly generate approximately valid \emph{candidate} CBFs, and refine these (marginally) to obtain a \emph{valid} CBVF. We proceed by developing the theoretical foundations that guarantee that the obtained  refined CBVF is valid upon convergence, while providing desired properties while converging. 

To facilitate both the practical implementation and leverage the theoretical guarantees of warmstarting previously developed for HJ reachability~\cite{HerbertBansalEtAl2019}, we first remark the following:
\begin{myprop}
[Forward completeness and safety with larger online discount rate]\label{prop:larger_online}
Applying a discount factor $\gi{j}$ online~\eqref{eq:online-cbf} for a CBVF that is constructed~\eqref{eq:cbvf} with $\gi{i} \leq \gi{j}$ maintains control invariance of the safe set.
\begin{proof}
By inspection of Eq~\eqref{eq:cbvf_constraintset}, if $\gi{j} \geq \gi{i}$, we have that $K_{\BVgi{i}}^{\gi{j}}(x,s) \supseteq K_{\BVgi{i}}^{\gi{i}}(x,s) \neq \emptyset \ \forall x \in \mathcal{C}_{\gi{i}}$ and $s \in [t, 0]$. Hence, applying a larger discount rate $\gi{j}$ online maintains pointwise feasibility. In addition, at the boundary of the viability kernel $x \in \partial \mathcal{S}$, we have that $\BVgi{i} = 0 \  \forall \gi{i}$, hence $K_{\BVgi{i}}^{\gi{j}}(x,s) = K_{\BVgi{i}}^{\gi{i}}(x,s)$. Combined, this assures the same safety guarantees are maintained when applying a larger discount rate $\gi{j}$ online.
\end{proof}

\end{myprop}

In this work, we will use $\gi{i} = 0$ for computing the CBVF and use $\gi{j} = \gamma \in \mathbb{R}_+$ to maintain forward invariance through the safety filter,~\eqref{eq:online-cbf}. In practice, we tune $\gamma$ based on the desired application. By considering $\gi{i}=0$ for the CBVF construction, we de facto consider the standard HJ reachability problem. As such, we can leverage the theory for HJ reachability warmstarting developed in~\cite{HerbertBansalEtAl2019}.
\subsection{Warmstarting reachability}
Instead of initializing the CBVF with the constraint function, $B_\ell(x, 0) = \ell(x)$, we initialize with a \emph{candidate} CBF $B_{h}(x,0)=h(x)$ (and drop the subscript of $B_0$ for ease of notation). We denote the converged value of the original reachability problem as $B_{\ell^*}(x) = \lim_{t \to -\infty} B_\ell(x,t)$ and the warmstarted problem as $B_{h^*}(x)$. We denote the 0-superlevel sets of these value functions at time $t<0$ as $\mathcal{C}_\ell(t)$ and $\mathcal{C}_h(t)$, with $\mathcal{C}_\ell(t) \equiv \mathcal{S}(t)$, the viability kernel, by definition. 

Following~\cite{HerbertBansalEtAl2019}, with slightly modified notation, we show that $\mathcal{C}_h(t)$ is a conservative under-approximation of $\mathcal{C}_\ell(t)$ $\forall t< 0$. 

\begin{mytheorem}[Conservativeness of warmstarting,~\cite{HerbertBansalEtAl2019}]\label{thm:classic}
    For any initialization of $B_h(x,0) = h(x)$, the value function at any time throughout the iterations until convergence will be pointwise conservative with respect to the exact solution, i.e. $\forall t < 0, B_h(x,t) \leq B_\ell(x,t) \  \forall x \in \mathcal{X}$.
\end{mytheorem}
Due to this inequality, the safe set (i.e. 0-superlevel set) of the warmstarted value function $B_h(x,t)$ is a subset of the safe set  with respect to the original reachability problem throughout the updates, i.e. $\mathcal{C}_{h}(t) \subseteq \mathcal{C}_{\ell}(t)$. This results in a conservative under-approximation of the true safe set. At convergence we have $\mathcal{C}_{h^*} \subseteq \mathcal{S}^*$, hence a converged CBVF is guaranteed to be valid.

Next we show that \refineCBF can be updated in-the-loop for robotic systems without the need for convergence, as each iteration of the algorithm will either maintain safety (if the candidate CBF is conservative with respect to the viability kernel) or become less unsafe (if the candidate CBF is invalid).
We show that throughout DP recursion (including at convergence), the current iteration's safe set $\mathcal{C}_{h}(t)$ is a subset of the union of the infinite-time viability kernel and the \emph{candidate} CBF, i.e. $\mathcal{C}_{h}(t) \subseteq \mathcal{S}^* \cup \mathcal{C}_{h}$.
In the special case that $\mathcal{C}_{h}$ is a subset (superset) of $\mathcal{S}^*$, every iteration of the refined CBF's safe set is a subset of the viability kernel (candidate CBF's 0-superlevel set).

\begin{mytheorem}\label{thm:moresafe}
   $B_h(x, t - \delta) \leq \max \{B_{\ell^*}(x), B_h(x,t)\} \  \forall x \in \mathcal{X}, \ \forall t\leq 0, \delta \in \mathbb{R}_+$, and in particular $B_h(x,t) \leq \max\{B_{\ell^*}(x), h(x)\}$.
\end{mytheorem}
\begin{proof}
We first provide two intermediate results (Lemma~\ref{lem:limited_contraction} and Lemma~\ref{lem:getting_safer}), which together help prove Theorem~\ref{thm:moresafe}.
\begin{mylem}[Limited contraction]\label{lem:limited_contraction}
    For all $x \in \mathcal{X}$ such that $B_h(x,0) = h(x) \leq B^*_\ell(x)$ we have that $B_h(x,t) \leq B^*_\ell(x)$ $\forall t < 0$.
\end{mylem}
\begin{proof}
    By contradiction. Suppose there exists $\tilde x$ and $\tilde t$ such that $B_h(\tilde x, \tilde t) > B_{\ell^*}(\tilde x)$. By the continuity of the update backwards in time there exists $\delta \in [0, -\tilde t]$ such that $B_h(\tilde x, \tilde t + \delta) = B_{\ell^*}(x)$. Since by definition this point is stationary (as $B_{\ell^*}(x)$ is the converged value), we have that $B_h(\tilde x, \tilde t + \delta + \epsilon) = B_h(\tilde x, \tilde t + \delta) = B_{\ell^*}(\tilde x)$ $\forall \epsilon < 0$, hence also for $\epsilon = -\delta$, resulting in $V(\tilde x, \tilde t) = B_{\ell^*}(\tilde x)$. This leads to a contradiction, and as such proves the lemma.
\end{proof}
Lemma~\ref{lem:limited_contraction} assures that any point that is initially conservative will remain conservative throughout the updates. In terms of characterizing the safe set, we have that for all $x$ such that $h(x) \leq B_{\ell^*}(x)$, $\mathcal{C}_h(t) \subseteq \mathcal{S}^*$, hence we are guaranteed that our safe set is contained within the viability kernel throughout the updates at said $x$, maintaining safety.

\begin{mylem}[Updates do not lead to more unsafety]\label{lem:getting_safer}
    For all $x\in \mathcal{X}$ such that $B_h(x,0) = h(x) \geq B^*_\ell(x)$, $B_h(x,t - \delta) \leq B_h(x,t)$ $\forall t \leq 0, \delta \in \mathbb{R}_+$
\end{mylem}
\begin{proof}
By definition of the dynamic programming recursion. By~\eqref{eq:dp}, $B_\ell(x, \tilde t - \delta) \leq B_\ell(x, \tilde t)$ $\forall \tilde t \leq 0, \delta \in \mathbb{R}_+$. This guarantees that the function will never decrease in value (i.e. relative safety level). For all such $x$, there exists $\tilde t$ such that $B_h (x,0) = B_\ell (x,\tilde t)$. Hence $B_h(x, t-\delta) \leq B_h(x, t)$ $\forall t \leq 0, \delta \in \mathbb{R}_+$, so $B_h(x,t)$ is guaranteed to maintain or contract to a higher value (and therefore safety level).
\end{proof}

Lemma~\ref{lem:getting_safer} assures that any point that is initially unsafe will not become more unsafe, i.e. obtain a lower value. In terms of characterizing the safe set, every iteration is contained in the prior, i.e. $\mathcal{C}_h(x, t - \delta) \subseteq \mathcal{C}_h(x, t)$ $\forall \delta \in \mathbb{R}_+$ at said $x$.

Combined, Lemma~\ref{lem:limited_contraction} and~\ref{lem:getting_safer} derive Theorem~\ref{thm:moresafe}.
\end{proof}

Theorem~\ref{thm:moresafe} implies that at every iteration of \refineCBF its safe set is contained in the union of the viability kernel and its previous iteration, i.e. $\mathcal{C}_h(t-\delta) \subseteq \mathcal{S}^* \cup \mathcal{C}_h(t)$ $\forall t \leq 0, \delta \in \mathbb{R}_+$. In particular this implies that its safe set is at all times contained within the safe set of the union of the viability kernel and the \emph{candidate} CBF. Lastly, if the \emph{candidate} CBF is contained within the viability kernel (Figure~\ref{fig:illustration}, top), then \refineCBF's safe set is contained within the viability kernel at all times or alternatively if the viability kernel is contained within the \emph{candidate} CBF then the safe set of each iteration of \refineCBF is contained within the previous safe set (Figure~\ref{fig:illustration}, bottom). Hence, our proposed method guarantees the CBVF to become monotonically more safe (or remain safe if already safe) with every DP iteration, and by Theorem~\ref{thm:classic} the CVBF is guaranteed to be valid upon convergence. As such, we can use this algorithm in-the-loop as with every update the CBVF is provably less unsafe.

Nonetheless, we point out that while converging we can only guarantee invariance with respect to the 0-superlevel set of $\ell(x)$ for a time $t$. However, we stress that the time-varying CBVF is at least as safe as using the candidate CBF.
\begin{myremark}[Pointwise feasibility of warmstarted value function]
    The CBVF at time $t$ can only guarantee that the system will stay within constraint set $\LL$ for the current time horizon $[t,0]$. This is because~\eqref{eq:cbf} includes a partial derivative of the CBVF with time while converging, i.e., $\dot h(x,t) = \frac{\partial h}{\partial t}(x,t) + L_f h(x,t) + L_g h(x,t) u$.  However, if the candidate CBF was safe, the refined CBVF will be safe over an infinite horizon throughout its iterations. In practice, one can often disregard the partial derivative with time when implementing the CBVF input constraint in a safety filter~\eqref{eq:online-cbf}.
\end{myremark}

\section{Practical implementation of \refineCBF}\label{sec:practice}
In this section, we describe how to leverage the theory developed above in practice for safety-critical autonomous systems. Specifically, we follow the following steps to refine a CBF (either online or offline):
\begin{enumerate}
    \item We spatially discretize the analytical CBF $h(x)$, requiring the evaluation of the CBF at each point of this grid.
    \item We apply \refineCBF, i.e.  CBVF DP updates iteratively, online in the loop or offline until convergence. 
    \begin{enumerate}
        \item Based on Proposition~\ref{prop:larger_online} we opt to use $\gamma=0$ for the CBVF update~\ref{eq:cbvf} and are free to tune $\gamma \geq 0$ in the safety filter~\eqref{eq:online-cbf} based on the desired application. A larger $\gamma$ results in a less invasive safety filter, i.e. allowing a faster decrease to the safety boundary.
        \item Leveraging warmstarting for HJ reachability with $\gamma=0$, we can initialize \refineCBF with the spatially discretized state. Theorem~\ref{thm:moresafe} guarantees that the updates do not become more unsafe (and converge to a valid CBVF) when applying the update recursively backward in time.
    \end{enumerate}
    \item Online, we can efficiently implement the safety filter~\eqref{eq:online-cbf} using spatial finite differences for computing the Lie derivatives at the grid points and interpolation to quantify the value function and its derivative.  
\end{enumerate}

\section{Simulation results}\label{sec:results}
We highlight the versatility of \refineCBF on a range of examples in simulation, covering a wide range of systems and different use cases for \refineCBF and provide an open-source implementation of our implementation at \href{https://github.com/UCSD-SASLab}{https://github.com/UCSD-SASLab/refineCBF}. All figures depict the obstacles as the complement of the constraint set, $\mathcal{L}^c$, in gray and the viability kernel, $\mathcal{S}^*$, in green. The evolving safe set boundary of \refineCBF is shown in several figures from dark to light gray and the converged safe set boundary is depicted in dark green. In practice, for all examples considered, \refineCBF recovers (a tight conservative approximation of) the viability kernel upon convergence. All examples consider a discount factor of $\gamma=5$ in the safety filter, Eq~\eqref{eq:online-cbf}, and depict the nominal policy $\hat \pi$.
\begin{figure}[!t]
    \vspace{0.25cm}
    \centering
    \includegraphics[width=\columnwidth]{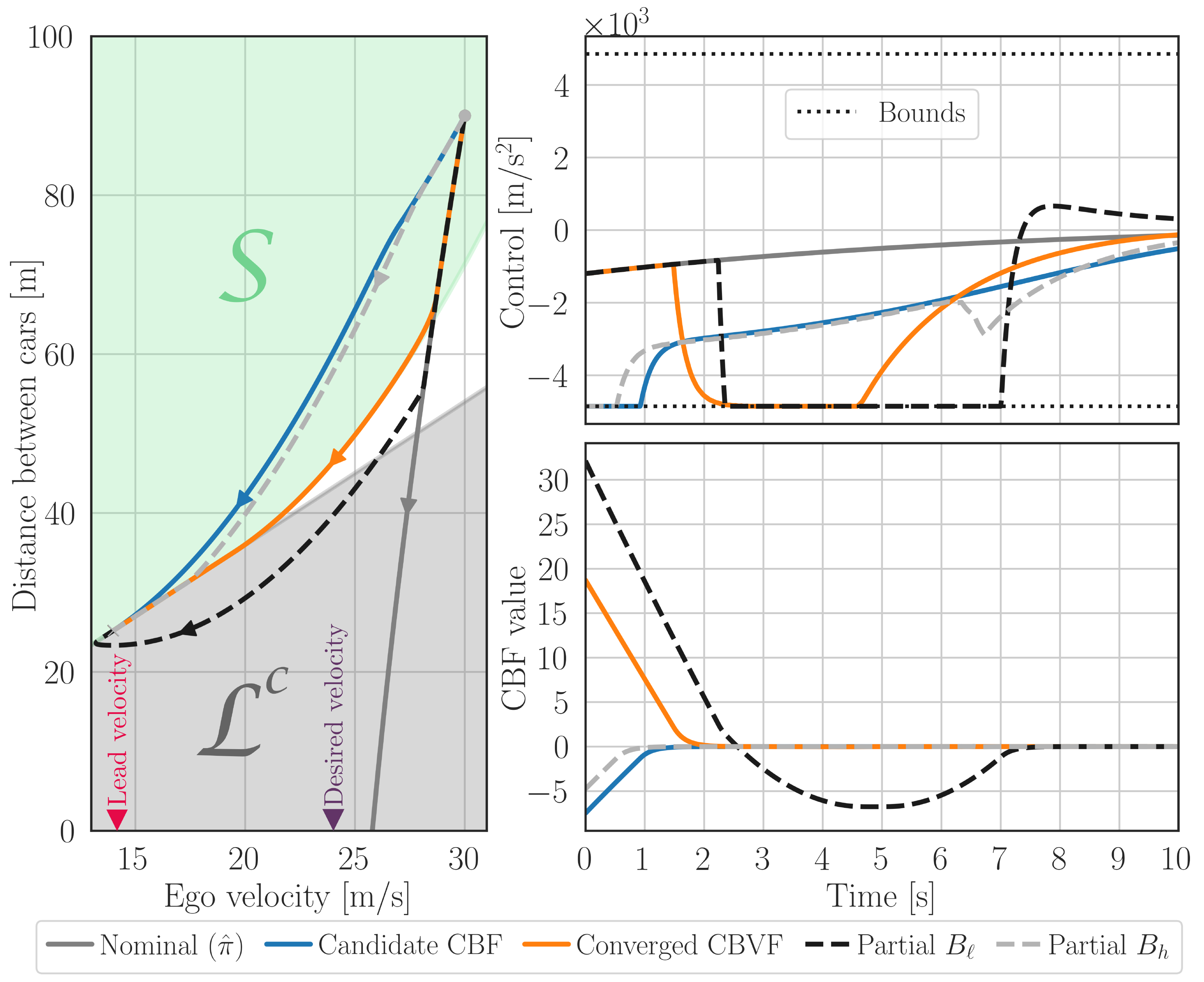}
    \caption{An adaptive cruise control system is tasked with maintaining a desired speed (24 m/s) while keeping a safe distance to a lead vehicle, subject to (de-)acceleration constraints. The CBVF-based safety filter can maintain the vehicle's desired speed for a longer period than the CBF-based safety filter achieving higher performance while maintaining safety guarantees. The partially converged warmstarted $B_h$ (safe) and standard CBVF $B_\ell$ (unsafe) demonstrate the issues with using vanilla reachability in the loop.}
    \label{fig:acc}
    \vspace{-0.3cm}
\end{figure}
\subsection{Adaptive Cruise Control: Accounting for friction}
An adaptive Cruise Control system involves an asymptotic performance objective (drive at a desired speed), safety constraints (maintain a safe distance from the car in front of you), and constraints based on the physical characteristics of the car and road surface (bounded acceleration and deceleration)~\cite{AmesXuEtAl2016}. For simplicity, we consider a scenario where the lead vehicle has a constant speed $v_0$, which is lower than the desired speed $v_d$ of the ego vehicle. The constraint function is defined as $\ell(x) = z - T_h v$ with $T_h$ the look-ahead time, $z$ the distance between the vehicles and $v$ the ego-velocity. We control acceleration with bounds $u \in [-m c_d g, mc_d g]$, with mass $m$, gravitational constant $g$ and $c_d$ the maximum g-force that should be applied. The dynamics are as follows: $\dot v = -\frac{1}{m}F_r(v) + \frac{1}{m}u$, $\dot z = v_0 - v$, with $F_r(v) = f_0 v^2 + f_1 v + f_2$ the friction (aerodynamic and rolling resistance). For a simplified model that neglects friction, we can compute its viability kernel analytically; It is defined as the 0-superlevel set of the valid CBF $h(x)~=~z~-~T_h v~-~\frac{(v_0 - v) ^ 2}{2c_dg}$. The CBF $h(x)$ is also guaranteed to be safe for the original dynamics as friction augments the action of braking, but is conservative, especially at higher speeds where friction has a significant effect on the dynamics (shown by the diverging boundaries of $\mathcal{S}^*$ and $\mathcal{L}^c$ in Figure~\ref{fig:acc}, left). We show that we can refine $h(x)$ to become less conservative, i.e. such that the safety filter is less invasive, with \refineCBF, and in practice, we recover the viability kernel upon convergence. The improved performance of the safety filter using $h^*(x)$ (converged CBVF) compared to the candidate CBF is shown in simulation in Figure~\ref{fig:acc}, right. The converged CBVF safety filter allows the driver to maintain a desired velocity for a longer period while guaranteeing safety and respecting the input constraints.

Additionally, we depict the partially converged CBVF $B_h$ (initialized with $h(x)$) and the partially converged CBVF $B_\ell$ (initialized with $\ell(x)$), to highlight that you can use $B_h$ in a safety filter (provided $h(x)$ is conservative), whereas $B_\ell$ leads to unsafety.

\begin{figure}[!t]
    \vspace{0.25cm}
    \centering
    \includegraphics[width=\columnwidth]{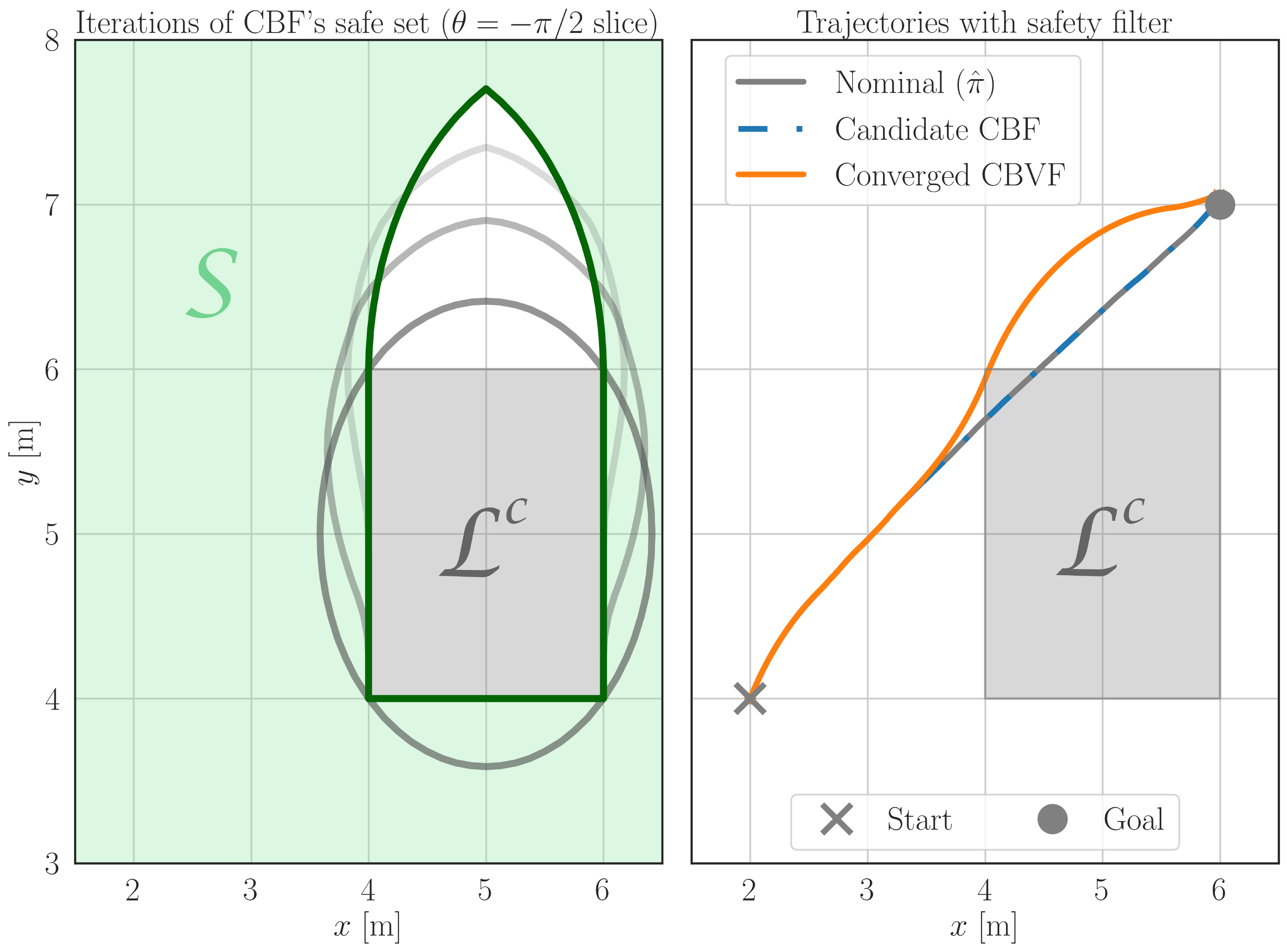}
    \caption{The 3D-Dubins car model with fixed velocity is tasked with getting to a target while avoiding a polyhedral obstacle. The CBF formed by over-approximating the obstacle as an ellipsoid (dark gray) is of relative degree 2 and cannot be used within a safety filter. In contrast, the converged CBVF (dark green) is a subset of the viability kernel. The CBVF constraint can be enforced by a safety filter and enables reaching the goal successfully without hitting the obstacle.}
    \label{fig:dubins}
    \vspace{-0.3cm}
\end{figure}
\subsection{Dubins Car: Polytopic constraints}
Constructing a CBF for non-holonomic systems is a challenge, especially in combination with polytopic constraints. We consider a Dubins car model: $\dot x = v \cos(\theta)$, $\dot y = v \sin(\theta), \dot \theta = u$, where $x, y$ are position, $\theta$ is heading, $v$ is a fixed speed, and we control the angular velocity $u \in [-0.5, 0.5]$. We consider the objective of navigating to a goal while having to navigate around an obstacle. In practice, polytopic constraints are usually approximated as points, paraboloids or hyper-spheres to ensure the distance function can be calculated explicitly as an analytic expression from their geometric configuration~\cite{ThirugnanamZengEtAl2021}. We approximate the obstacle with an ellipse, see Figure~\ref{fig:dubins} (left). Additionally, for a purely positional CBF, $h = h(x,y)$, the Dubins car model has relative degree $2$, and a standard CBF cannot provide safety guarantees. While an exponential CBF can be synthesized for higher relative degree systems~\cite{NguyenSreenath2016}, we highlight how \refineCBF can be used to obtain a relative degree 1 CBVF from a higher-relative degree candidate CBF. The convergence of the CBVF over time from the CBF is shown in Figure~\ref{fig:dubins} (left). At convergence, we recover the viability kernel. Additionally, in Figure~\ref{fig:dubins} (right) we show how the converged CBVF can be used to guarantee safety, unlike the candidate CBF.

\subsection{Quadrotor: Expert synthesized CBFs}
In practice, CBFs are often synthesized as follows: (1) the dynamics are linearized around an operating point (e.g. the mid-point of the operating region), (2) an LQR feedback law $K$ is devised such that the closed-loop system $\dot{x} = (A + BK)u$ is stable, (3) a Lyapunov function $V(x)$ is constructed from the Lyapunov equation for the closed-loop system, (4) the Lyapunov function is mirrored in the state plane and a constant $c > 0$ is added to construct the candidate CBF; $h(x) = c - V(x)$. However, these barriers do not capture the nonlinearities of the model, are difficult to tune and do not provide safety guarantees. Yet they often still provide a good approximation of the safe set. Thus, when locally refined using \refineCBF we can obtain safety guarantees. We consider a self-righting 6-state planar quadrotor model adopted from~\cite{SinghRichardsEtAl2020} that upon being deployed stabilizes (close) to a desired hover position and we consider constraints on the vertical position (not crashing into the ground and the ceiling). In reality, close to the ground (ceiling), a large positive (negative) vertical velocity is safe, but hand-crafting a continuously differentiable CBF to encode this is difficult.
\begin{figure}[!t]
    \vspace{0.25cm}
    \centering
    \includegraphics[width=\columnwidth]{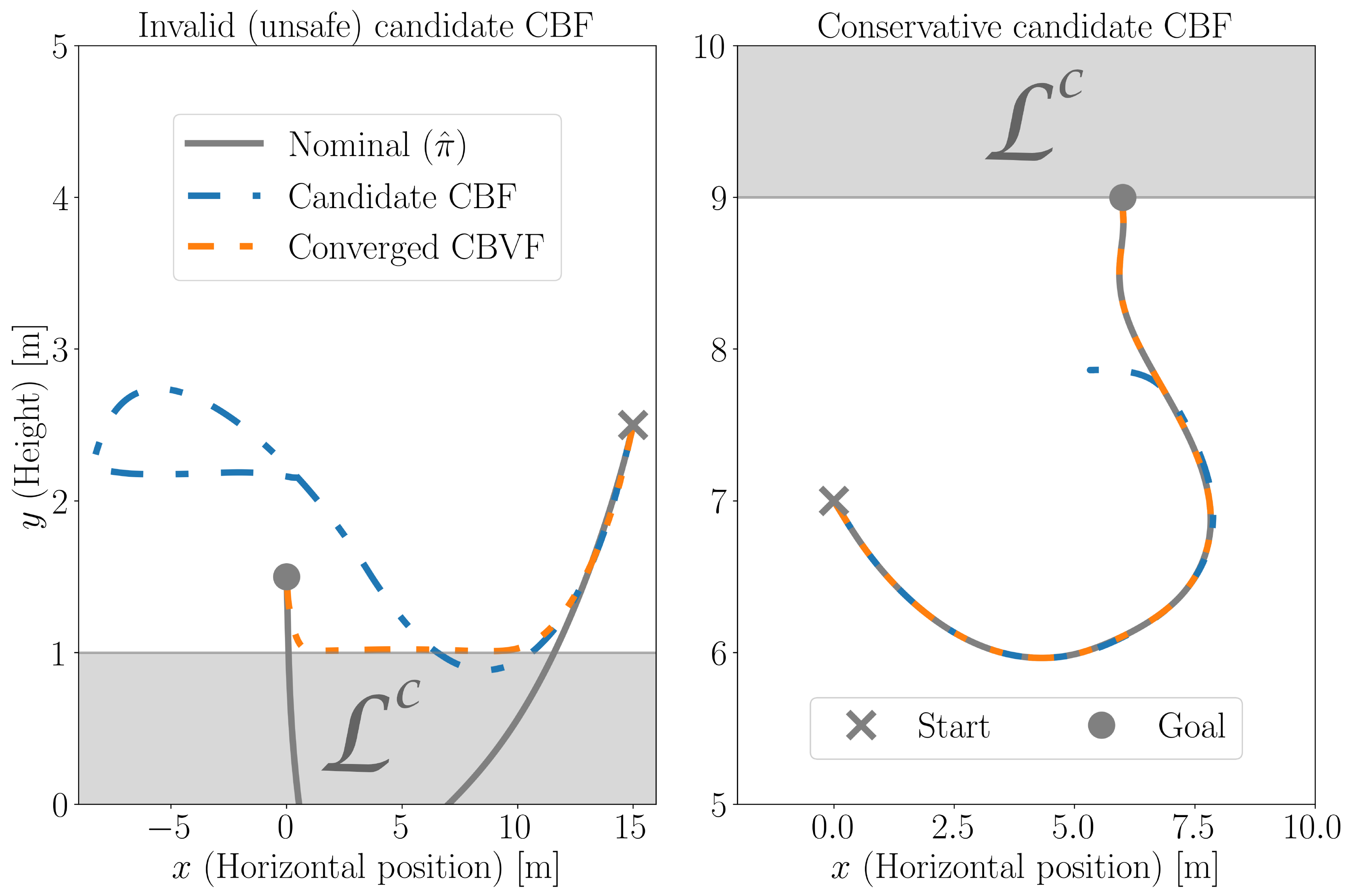}
    \caption{A self-righting planar quadrotor is tasked with stabilizing to a desired stationary point in the $(x,y)$ plane. The safety filter attempts to avoid collision with the obstacles. The expert-synthesized CBF leads to unsafety (left) for some tasks, while hindering in reaching a safe goal (right) for others. In contrast, the converged CBVF maintains safety and achieves higher performance.}
    \label{fig:quad}
    \vspace{-0.4cm}
\end{figure}
To aid with the curse of dimensionality we consider a 4-state quadrotor that disregards the horizontal position $x$ and velocity $v_x$ for refining the CBF. We design an expert synthesized candidate CBF to maintain ground and ceiling clearance and penalize high velocity and angular velocity to avoid having a CBF of relative degree $2$. We showcase how updating a candidate CBF can lead to higher performance (reaching a desired hover position faster) and safety (not crashing) in Figure~\ref{fig:quad}. Additionally, Figure~\ref{fig:illustration} highlights the difficulty of synthesizing a CBF that is valid in the entire state space and depicts the evolution of the safe set of \refineCBF in time.

\begin{figure}[!t]
    \vspace{0.2cm}
    \centering
    \includegraphics[width=\columnwidth]{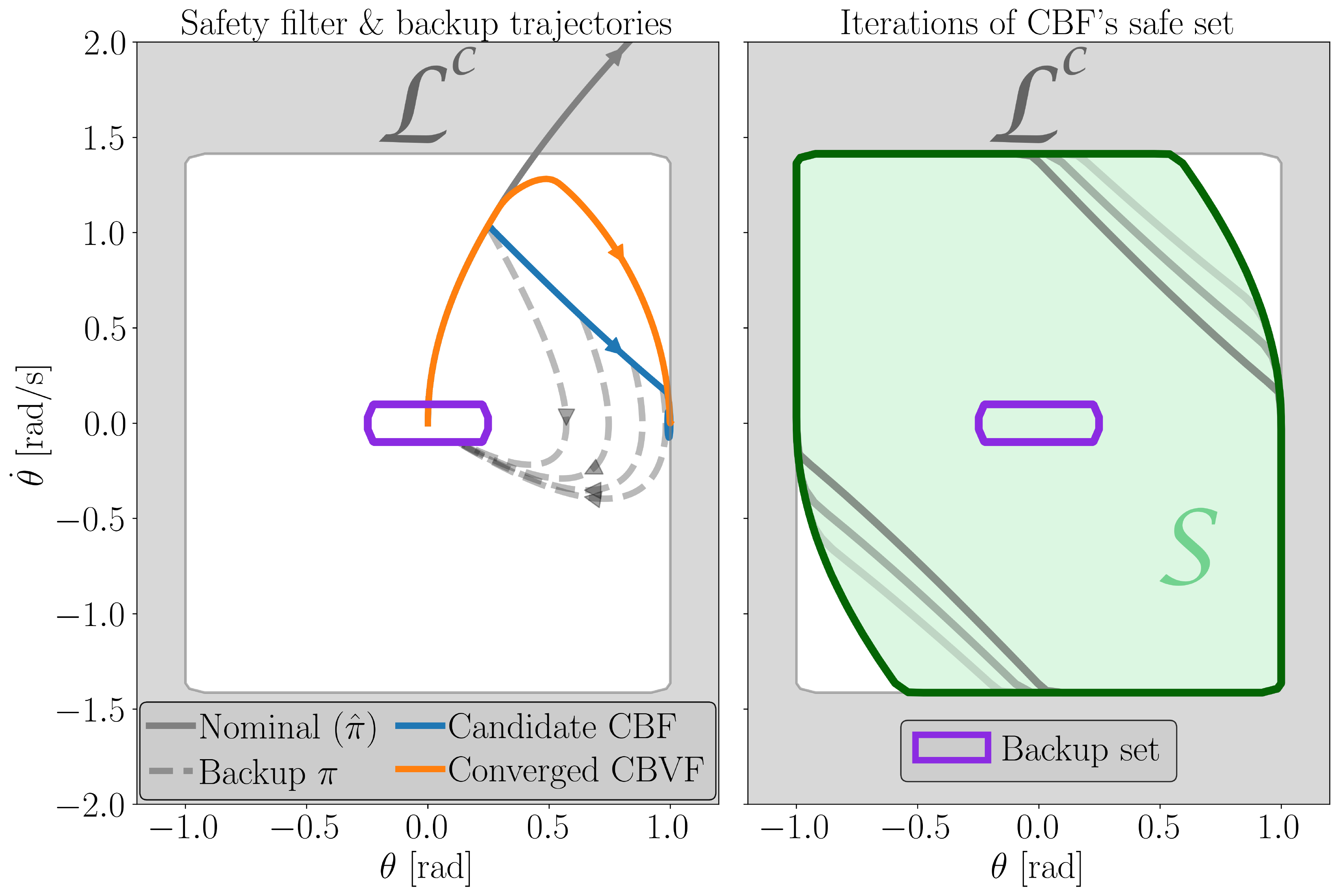}
    \caption{An implicit backup CBF (left) provides safety by assuring the trajectory remains within reach (dashed trajectories) of a backup set (purple). Policy evaluation is used to obtain an explicit representation (right, gray) and enables improving performance as \refineCBF converges (to dark green) while maintaining the safety guarantees of the backup CBF approach.}
    \label{fig:inv_pend}
    \vspace{-0.4cm}
\end{figure}
\subsection{Inverted Pendulum: Backup CBF}
An alternative approach to constructing a valid viable set leverages the backup CBF scheme (see Section~\ref{sec:relevance}, use case 4). However, enforcing the implicit CBF requires either imposing many CBF-like constraints in a safety filter or using a hand-tuned performance-safety trade-off controller. Instead, with \refineCBF, we can obtain an explicit representation of the valid candidate, yet implicit, CBF and enforce the CBF condition through a single constraint. In practice, after spatially discretizing the backup CBF, we perform policy evaluation~\cite{SuttonBarto1998} using the backup controller and its predefined fixed time horizon $T$ to compute the value of the backup controller at each state. Given that the backup-based CBF approach defines a valid, yet conservative, CBF, we obtain an explicit representation of a value function describing a control invariant set. We can then further improve (i.e. render less conservative) the CBF, while guaranteeing safety throughout convergence. Unlike derivative-free approaches~\cite{SingletarySwannEtAl2022}, \refineCBF can be used with a principled safety filter and comes with natural robustness (see Remark~\ref{rem:cbf_unsafe}).

We consider the inverted pendulum example from~\cite{SingletarySwannEtAl2022}. Figure~\ref{fig:inv_pend} (left) showcases how an implicit CBF can be used to enforce safety through a backup policy (gray) using a trade-off controller (candidate CBF). Figure~\ref{fig:inv_pend} (right) demonstrates how the explicit representation of this valid CBF can be refined (dark to light gray) with \refineCBF to recover (dark green) the viability kernel while retaining safety throughout convergence.\vspace{0.4cm}
\addtolength{\textheight}{-4.4cm}

\section{Conclusions}\label{sec:conclusion}
In this work we propose \refineCBF, an algorithm that leverages HJ reachability, a constructive method, to refine candidate CBFs. We guarantee that the updated CBVF is at least as safe as the candidate CBF throughout the convergence of \refineCBF -- and obtain a valid CBVF upon convergence. We demonstrated \refineCBF through simulations in which we demonstrate that our CBF refinement approach is applicable for a broad class of candidate CBFs and demonstrate a simple, yet powerful method to improve the performance and/or safety of CBF-based safety filters.

\textbf{Future work:}
First, we plan on validating \refineCBF on hardware experiments to validate that it can be used in-the-loop and demonstrate its utility when the system is faced with different operating conditions (e.g. larger wind disturbance) that may violate initial assumptions. Additionally, we are interested in leveraging this method to locally refine boundary regions of learned CBFs that cannot be verified using e.g., adversarial sampling or formal verification methods.  

\section*{Acknowledgements}
We would like to thank Jason J. Choi and Sicun Gao for the insightful discussions.

\bibliographystyle{IEEEtran}
\bibliography{main, ASL_papers, SASLab}
\end{document}